\documentclass[article]{IEEEtran}
\usepackage{microtype}
\usepackage{graphicx}
\usepackage{subcaption}
\usepackage{booktabs}
\usepackage{hyperref}
\usepackage{amsmath, amssymb, amsfonts}
\usepackage{mathtools}
\usepackage{amsthm}
\usepackage{array}
\usepackage{multirow}
\usepackage{tabularx}
\usepackage{threeparttable}
\usepackage{algorithm}
\usepackage{algorithmic}
\usepackage{url}
\usepackage{textcomp}
\usepackage{xcolor}
\usepackage{siunitx}
\theoremstyle{plain}
\newtheorem{theorem}{Theorem}[section]
\newtheorem{proposition}[theorem]{Proposition}

\theoremstyle{definition}
\newtheorem{definition}[theorem]{Definition}

\theoremstyle{remark}

\usepackage{amsthm}

\usepackage{natbib}

\begin{document}

\author{Randolph Wiredu-Aidoo}

\title{The Mass Agreement Score: A Point-centric Measure of Cluster Size Consistency}

\maketitle

\begin{abstract}
In clustering, strong dominance in the size of a particular cluster is often undesirable, motivating a measure of cluster size uniformity that can be used to filter such partitions. A basic requirement of such a measure is stability: partitions that differ only slightly in their point assignments should receive similar uniformity scores. A difficulty arises because cluster labels are not fixed objects; algorithms may produce different numbers of labels even when the underlying point distribution changes very little. Measures defined directly over labels can therefore become unstable under label-count perturbations. I introduce the \textit{Mass Agreement Score (MAS)}, a point-centric metric bounded in $[0, 1]$ that evaluates the consistency of expected cluster size as measured from the perspective of points in each cluster. Its construction yields fragment robustness by design, assigning similar scores to partitions with similar bulk structure while remaining sensitive to genuine redistribution of cluster mass.
\end{abstract}

\section{Introduction}

Clustering algorithms often produce partitions whose clusters vary widely in size. In many applications, however, strong dominance by a single cluster is undesirable, as the algorithm has largely failed to uncover structure beyond the trivial single-cluster solution. For this reason, cluster size uniformity is often used as a heuristic signal when evaluating or filtering clustering outputs. A useful uniformity measure should therefore assign high scores to partitions whose clusters carry comparable mass and low scores to partitions dominated by a single cluster.

Beyond this basic ordering, a uniformity measure should also behave smoothly: partitions that differ only slightly in their point assignments should receive similar scores. This stability requirement is natural for any quantitative measure. In clustering, it is particularly important because algorithms may produce partitions that differ only by the reassignment of a small number of points, or by the subdivision of small residual clusters, while leaving the bulk structure essentially unchanged.

A challenge arises because clustering labels are not fixed objects. Unlike supervised classification, cluster identities carry no intrinsic meaning and their number may vary between partitions. Small changes in point assignments can therefore produce different numbers of cluster labels even when the underlying point distribution changes very little. Measures defined directly over cluster labels implicitly treat labels as fundamental objects and weight them equally, which can lead to instability when label counts fluctuate. In such cases, partitions with nearly identical bulk structure may receive widely different scores simply because small clusters are subdivided or merged.

In contrast, similarity measures commonly used in clustering evaluation (such as the Adjusted Rand Index) operate on relationships between points rather than on cluster labels themselves. Because they aggregate over point pairs, their behavior is naturally weighted by cluster mass: clusters containing few points contribute little to the overall measure. This perspective suggests that uniformity should likewise be evaluated from the standpoint of points rather than labels.

I introduce the \emph{Mass Agreement Score} (MAS), a bounded measure of cluster size uniformity derived from a point-centric, self-consistency perspective. Rather than asking whether cluster labels are uniformly distributed, MAS averages the size-consistency between each point's cluster size and the expected cluster size of the non-member distribution. This construction aligns the notion of uniformity with the mass-weighted structure that governs similarity between partitions.

The resulting measure exhibits several useful properties. MAS is bounded in $[0,1]$, attains its maximum when cluster masses are equal, and is stable under repartitioning of negligible mass. Consequently, partitions sharing the same bulk structure occupy a narrow score neighborhood, while genuine redistribution of mass among large clusters results in a significant movement of the score.

\section{Related Work}

\subsection{Alternative Measures}

\paragraph{Entropy-based measures.}
A natural approach to quantifying how evenly mass is distributed across clusters is Shannon entropy~\cite{shannon1948},
\[
H = -\sum_i p_i \log p_i,
\]
and its R\'enyi generalizations~\cite{renyi1961}. These measures capture dispersion over cluster labels and are maximized by the uniform distribution. However, their extrema depend explicitly on the number of clusters $K$. As a result, normalized entropy scores can shift when negligible clusters are split or merged, while unnormalized entropy grows in scale simply due to higher cluster count rather than genuine increase in uniformity. 

\paragraph{Concentration indices and the HHI.}
The Herfindahl--Hirschman Index (HHI)~\cite{rhoades1993}, computed as $\sum_j (n_j/N)^2$ for cluster sizes $n_j$, is widely used in economics, ecology (as the Simpson index), and political science. It is robust to fragmentation of negligible mass: a cluster of size $\varepsilon N$ affects the index by at most $O(\varepsilon)$. However, the minimum value of HHI depends on $K$, attaining $1/K$ under perfect uniformity. Normalized variants,
\[
\mathrm{HHI}^* = \frac{\mathrm{HHI} - 1/K}{1 - 1/K},
\]
restore a $[0,1]$ scale but retain this dependence: as $K$ grows through fragmentation of small clusters, $1/K \to 0$ and $\mathrm{HHI}^*$ approaches the unnormalized index. 

\paragraph{Inequality measures.}
The literature on income inequality provides a broad class of dispersion measures, including the Gini coefficient~\cite{gini1912}, Atkinson indices \citep{atkinson1970}, and generalized entropy measures. These operate on the full distribution of values across individuals. They can be used to measure a notion of uniformity through inversion (e.g. $1 - \mathrm{Gini}$), however they retain sensitivity to the category count even under movement of negligible mass. This is a desirable property for the fields within which they are often used, but is misaligned with the unstable nature of label count in clustering.

\paragraph{Cluster validity indices.}
Internal validation criteria such as Silhouette width \citep{rousseeuw1987}, the Dunn index \citep{dunn1974}, and the Davies--Bouldin index~\cite{davies1979} evaluate clusterings using geometric notions of compactness and separation. They do not treat size uniformity as a primary objective. As such, they may reward partitions that are geometrically coherent but uninformative due to the fusion of a vast majority of points into a single cluster. This motivates a degree of uniformity in cluster sizes as secondary constraint on optimal partitions.

\paragraph{Size-biased sampling.}
The perspective underlying MAS aligns with classical results on size-biased (or length-biased) sampling \citep{patilrao1978, cox1969}. When sampling individuals rather than groups, larger groups are encountered with probability proportional to their size, a manifestation of the inspection paradox. The quantity
\[
S = \sum_j \frac{n_j^2}{N}
\]
is precisely the mean group size under this sampling scheme. MAS adopts this viewpoint as its baseline, yielding a measure whose robustness to fragmentation follows directly from point-level weighting.

\subsection{Positioning MAS}

MAS is most closely related to the HHI measure, similarly leveraging size-biased sampling to enforce a point-centric perspective. In particular, MAS uses size-biased sampling of cluster size,
\[
S = \sum_j \frac{n_j^2}{N},
\]
to represent typical cluster size as experienced by a uniformly sampled point.

MAS then departs from this baseline in two key ways: First, it replaces the global quantity with a leave-one-out version
\[
S_i = \frac{\sum_{j} n_j^2 - n_i^2}{N - n_i},
\]
so that each cluster is evaluated against a size baseline absent of their own influence. Second, it aggregates agreements between cluster sizes and leave-one-out baselines at the level of points, weighting by $n_i/N$. This results in a uniformity measure that is insensitive to fragmentation of negligible mass by construction.

In contrast to entropy and inequality measures, MAS does not depend explicitly on the number of clusters, and unlike geometric validity indices, it operates purely on cluster sizes. 

\section{Proposed Method}

\subsection{Setup}

Consider a dataset of $N$ points partitioned into $K$ clusters with sizes $n_1, \dots, n_K$, satisfying
\[
  \sum_{i=1}^{K} n_i = N.
\]

\subsection{Point-Centric Baseline}

Select a point uniformly at random. The probability that it belongs to cluster $j$ is $n_j / N$, and it experiences a cluster of size $n_j$. The expected cluster size experienced by a random point is therefore
\[
  S = \sum_{j=1}^{K} \frac{n_j}{N} \cdot n_j = \frac{\sum_{j=1}^{K} n_j^2}{N}.
\]
This quantity reflects the typical cluster size as encountered by points rather than by cluster labels.

\subsection{Leave-One-Out Baseline}

To assess cluster $i$ for consistency against the rest of the distribution, a comparison baseline is constructed excluding cluster $i$. Conditioning on the remaining $N - n_i$ points, the expected cluster size they experience is
\[
  S_i = \frac{\sum_{j=1}^{K} n_j^2 - n_i^2}{N - n_i}.
\]
This value can be computed efficiently using a single precomputed global sum $\sum_j n_j^2$ and a per-cluster adjustment.

\subsection{Definition}

\begin{definition}[Mass Agreement Score]
For each cluster $i$, the normalized disagreement between its size and the leave-one-out baseline is
\[
  \frac{|n_i - S_i|}{N}.
\]
Aggregating point-centrically, the \emph{Mass Agreement Score} is
\[
  U = \sum_{i=1}^{K} \frac{n_i}{N} \left(1 - \frac{|n_i - S_i|}{N}\right).
\]
\end{definition}

\subsection{Range and boundary case}
For any partition with at least two nonempty clusters, one has $0<n_i<N$ and
\[
  S_i = \frac{\sum_{j\ne i} n_j^2}{N-n_i}.
\]
Since the complementary cluster sizes are positive and sum to $N-n_i$,
\[
  \sum_{j\ne i} n_j^2 \le (N-n_i)^2,
\]
implying $0 < S_i \le N-n_i < N$. Consequently $|n_i - S_i| < N$, so each per-cluster disagreement term satisfies
\[
  0 \le \frac{|n_i - S_i|}{N} < 1,
\]
and each agreement term lies in $(0,1]$. Because the weights $n_i/N$ are positive and sum to $1$, it follows that
\[
  U \in (0,1]
\]
for all partitions with $K \ge 2$.

The trivial one-cluster partition lies outside the domain of the leave-one-out baseline because the complementary distribution is empty. Since the attainable range for nontrivial partitions is $(0,1]$, we extend the definition by convention as $U=0$ for $K=1$, which naturally occupies the vacant endpoint of the scale. The extension is also consistent with the limiting behavior of highly
dominant partitions: for example, for $(N-1,1)$ one has $U \to 0$
as $N \to \infty$ (Section~\ref{prop:2_cluster_dominance}).

\begin{figure}
    \centering
    \includegraphics[width=0.9\linewidth]{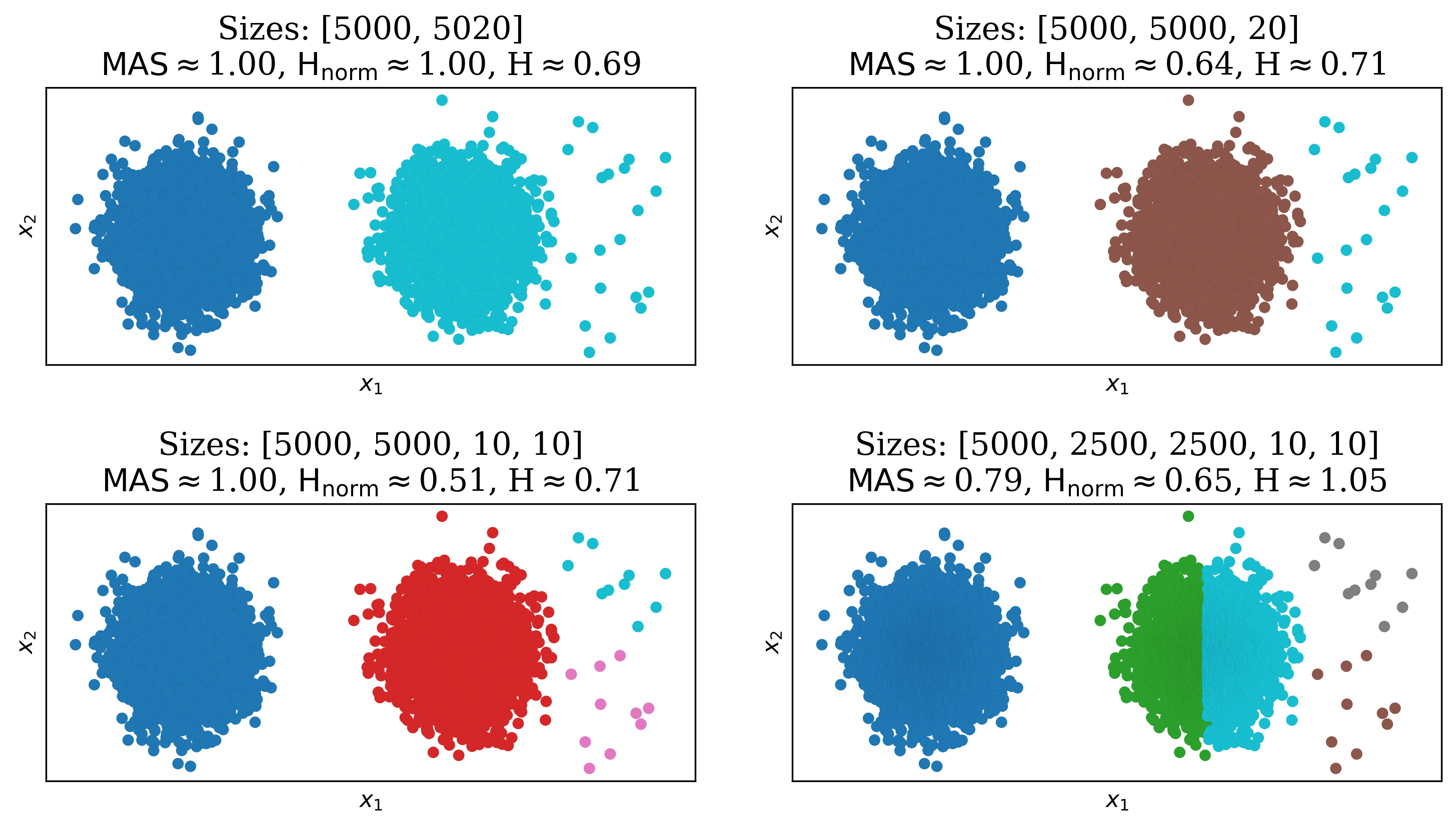}
    \caption{MAS and entropy ($\mathrm{H}$) values across four partitions of $N = 10{,}020$. MAS remains near $1.00$ across partitions with the same bulk structure, significantly reducing only when a large cluster ($\approx 0.5 N$) is split in half. Normalized entropy instead falls from $1.00$ to $0.51$ across similar partitions and rises to its second-highest value when the large cluster is split, demonstrating a label-centric focus. Unnormalized entropy primarily increases with label count, reflecting the larger entropy scale induced by an increase in label count.}
\end{figure}

\subsection{Interpretation}
The Mass Agreement Score measures the consistency of the cluster size distribution through the perspective of data points in each cluster. It averages the similarity between a point's cluster size and the average cluster size experienced by a non-member point, creating a statistic on cluster size consistency as measured through each data point. The agreement term 
\[ 
1 - \frac{|n_i - S_i|}{N} 
\] 
captures the notion of size similarity as an inverted $L_1$ discrepancy, normalized into a $(0, 1]$ range by the factor of $N$.

\subsection{Time Complexity}

Computation of MAS is $O(K)$ given cluster sizes. Three $O(K)$ passes suffice: (1) accumulate $Q = \sum_j n_j^2$; (2) compute each $S_i = (Q - n_i^2)/(N - n_i)$ in $O(1)$; (3) accumulate the weighted sum defining $U$. If cluster sizes must first be tallied from point-level assignments the cost is $O(N)$, matching a single scan over the data.

\subsection{Key Properties}
I prove key properties of the Mass Agreement Score, including its maximization at perfect uniformity of sizes, scale invariance, and stability under point reassignment in Appendix~\ref{app:key_props}.

\section{Experiments}

\subsection{Setup}

To characterize the behavior of MAS relative to alternative
measures, I conduct a series of experiments: two synthetic
experiments designed to isolate particular distributional regimes,
and one evaluation of uniformity measures embedded in a composite
hyperparameter scoring pipeline across eight datasets.

\paragraph{Experiment 1: progressive fragmentation of the small cluster.}
We begin with two clusters of size $4950$ and a small cluster of $100$. The small cluster is successively split into $2^i$ near-equal pieces for $i = 0, 1, \dots, 6$, and finally into 100 singletons. When 100 is not divisible by $2^i$, the remainder is distributed across the groups one point at a time. The two bulk clusters remain unchanged throughout.

\paragraph{Experiment 2: splitting a bulk cluster.}
Starting from the base partition, $[4950, 4950, 100]$, one of the large clusters is split in half, yielding cluster sizes $[4950, 2475, 2475, 100]$. We then consider splitting both large clusters, giving $[2475, 2475, 2475, 2475, 100]$. For comparison, the same splits are also evaluated without the small cluster present.

\paragraph{Experiment 3: uniformity measures in a composite hyperparameter scoring pipeline.}
To evaluate the practical utility of MAS in a realistic hyperparameter
selection setting, I embed each uniformity measure inside a composite
scoring function and assess how well the resulting scores rank clustering
outputs by quality. For a partition with uniformity score $v$, effective
cluster count $K_{\mathrm{eff}} = 1/\mathrm{HHI}$~\cite{laakso1979}, and silhouette score
$\mathrm{Sil} \in [-1, 1]$, the composite score is defined as
\begin{equation}
    \Phi(v) \;=\; v \cdot \left(1 - \log_N K_{\mathrm{eff}}\right)
    \cdot \frac{\mathrm{Sil} + 1}{2}.
\end{equation}
The factor $(1 - \log_N K_{\mathrm{eff}})$ penalizes excessive
fragmentation and $(\mathrm{Sil} + 1)/2$ rewards geometric compactness
and separation; the uniformity term $v$ is the sole point of variation
across conditions. I apply Spectral Clustering to eight datasets spanning
two to sixty dimensions: \textit{Aggregation} ($2\mathrm{D}$),
\textit{Moons} ($2\mathrm{D}$), \textit{Unbalance} ($2\mathrm{D}$),
\textit{Iris} ($4\mathrm{D}$), \textit{Banknote Authentication}
($4\mathrm{D}$), \textit{Wine} ($13\mathrm{D}$), \textit{WDBC}
($30\mathrm{D}$), and \textit{Sonar} ($60\mathrm{D}$), obtained from the Gagolewski repository~\cite{gagolewski2022framework}. For each dataset, I apply standardization and
sweep $k \in \{2, \dots, \max(\lfloor \log_2 N \rfloor, 10)\}$ with
$k$-nearest-neighbor affinity and $n_{\text{neighbors}} = \lfloor
\sqrt{N} \rfloor$. Critically, the ground-truth partition is included
among the candidates on each dataset. This ensures that the correct
partition is always attainable, so that partition ranking differences reflect the scorer's ability to identify quality partitions
rather than the absence of them from the candidate set.
Scoring quality is measured by the Pairwise Ranking Similarity (PWRS):
the fraction of partition pairs on which $\Phi(v)$ and the ARI agree in
their relative ordering, excluding tied pairs,
\begin{equation}
    \mathrm{PWRS} = \frac{
        \#\left\{(i,j) : i < j,\;
        (\mathrm{ARI}_j - \mathrm{ARI}_i)(\Phi_j - \Phi_i) > 0
        \right\}
    }{
        \#\left\{(i,j) : i < j,\;
        \mathrm{ARI}_j \neq \mathrm{ARI}_i \text{ and }
        \Phi_j \neq \Phi_i
        \right\}
    }.
\end{equation}
A PWRS of $1$ indicates perfect agreement with the ARI ranking; $0.5$
corresponds to chance. As a reference condition I include $v \equiv 1$
(Null Reference), which measures the baseline effectiveness of the
$K_{\mathrm{eff}}$-silhouette combination alone and against which any
uniformity measure must improve to justify its inclusion.

\paragraph{Included metrics.} In each experiment, the following measures are computed:

\begin{itemize}
\item Shannon entropy
\[
H = -\sum_i p_i \ln p_i
\]
and its normalized form $H / \ln K$.

\item Rényi-2 entropy
\[
H_2 = -\ln \sum_i p_i^2
\]
and its normalized form $H_2 / \ln K$.

\item The Herfindahl--Hirschman Index
\[
\mathrm{HHI} = \sum_i p_i^2,
\]
reported as a uniformity score $1 - \mathrm{HHI}$.

\item The normalized HHI
\[
\mathrm{HHI}^* = \frac{\mathrm{HHI} - \frac{1}{K}}{1 - \frac{1}{K}},
\]
reported as $1 - \mathrm{HHI}^*$.

\item The Gini coefficient 
\[
  \mathrm{Gini} = \frac{\displaystyle\sum_{i=1}^{K}\sum_{j=1}^{K} |n_i - n_j|}{2KN}
\]
inverted as $1 - \mathrm{Gini}$:

\item The Mass Agreement Score (MAS).
\end{itemize}

\subsection{Results}

\subsubsection*{\textbf{Experiment 1}}

\begin{table*}[ht]
\centering
\caption{Experiment 1: all uniformity measures under progressive fragmentation of the small cluster.
The two bulk clusters of size $4950$ remain fixed throughout; only the $100$-point cluster is split.}
\label{tab:exp1}
\begin{tabular}{lrrrrrrrrr}
\toprule
& & & \multicolumn{2}{c}{Entropy} & \multicolumn{2}{c}{R\'{e}nyi-2} & \multicolumn{2}{c}{HHI} & \\
\cmidrule(lr){4-5}\cmidrule(lr){6-7}\cmidrule(lr){8-9}
Configuration & $K$ & MAS & $H$ & $H/\ln K$ & $H_2$ & $H_2/\ln K$ & $1-\mathrm{HHI}$ & $1-\mathrm{HHI}^*$ & $1-\mathrm{Gini}$ \\
\midrule
$2^0 = 1$ piece   &   3 & 0.9856 & 0.7422 & 0.6756 & 0.7130 & 0.6490 & 0.5099 & 0.7648 & 0.6767 \\
$2^1 = 2$ pieces  &   4 & 0.9855 & 0.7491 & 0.5404 & 0.7131 & 0.5144 & 0.5099 & 0.6799 & 0.5100 \\
$2^2 = 4$ pieces  &   6 & 0.9855 & 0.7561 & 0.4220 & 0.7132 & 0.3980 & 0.5099 & 0.6119 & 0.3433 \\
$2^3 = 8$ pieces  &  10 & 0.9854 & 0.7630 & 0.3314 & 0.7132 & 0.3097 & 0.5099 & 0.5666 & 0.2098 \\
$2^4 = 16$ pieces &  18 & 0.9854 & 0.7699 & 0.2664 & 0.7132 & 0.2468 & 0.5099 & 0.5399 & 0.1208 \\
$2^5 = 32$ pieces &  34 & 0.9854 & 0.7768 & 0.2203 & 0.7132 & 0.2023 & 0.5099 & 0.5254 & 0.0685 \\
$2^6 = 64$ pieces &  66 & 0.9854 & 0.7833 & 0.1870 & 0.7132 & 0.1702 & 0.5099 & 0.5178 & 0.0388 \\
$100$ singletons  & 102 & 0.9854 & 0.7883 & 0.1704 & 0.7132 & 0.1542 & 0.5099 & 0.5150 & 0.0296 \\
\bottomrule
\end{tabular}
\end{table*}

The numerical results are reported in Table~\ref{tab:exp1}.

\paragraph{MAS.}
MAS is essentially unchanged across the entire sequence, moving only from $0.9856$ to $0.9854$. This is consistent with the design of the score: the two dominant clusters remain fixed, and the fragmented cluster contributes only a small fraction of the total mass.

\paragraph{R\'enyi-2 entropy and unnormalized HHI.}
The raw R\'enyi-2 entropy $H_2$ and the unnormalized score $1-\mathrm{HHI}$ are also nearly invariant, with $1-\mathrm{HHI}$ remaining constant up to four decimal places. Here, the value of $1-\mathrm{HHI}$ sits near 0.5, reflecting the effective cluster count of $K = 2$ primary clusters, rather than a degree of uniformity. The insensitivity of these measures is expected as they are determined by $\sum_i p_i^2$, which is dominated by the two clusters of size $4950$. Splitting the $1\%$ cluster changes that quadratic mass only marginally.

\paragraph{Shannon entropy.}
Raw Shannon entropy $H$ increases modestly, from $0.7422$ to $0.7883$, because it is sensitive to the introduction of additional low-probability categories. This increase is not large in absolute terms, but it reflects the fact that entropy treats fragmentation itself as added diversity.

\paragraph{Normalized measures and Gini.}
The normalized scores $H/\ln K$ and $H_2/\ln K$ decline steadily as $K$ increases, even though the bulk mass distribution is essentially unchanged. This behavior is driven by the growing denominator $\ln K$: the scale of the maximum attainable value grows as the number of clusters increases.

A similar effect appears for the normalized HHI score $1-\mathrm{HHI}^*$, which falls from $0.7648$ to $0.5150$. Here the change is induced by the $K$-dependent baseline $1/K$, which causes the score to converge toward the unnormalized $1-\mathrm{HHI}$ in the limit.

The score $1-\mathrm{Gini}$ decreases most sharply, from $0.6767$ to $0.0296$, because the cluster-size vector becomes increasingly unequal in count as the small cluster is split into many tiny pieces.

\paragraph{Summary of Experiment 1.}
Experiment~1 shows that MAS, raw R\'enyi-2 entropy, and unnormalized HHI are largely insensitive to fragmentation of negligible mass, whereas the normalized entropies, normalized HHI, and Gini-based score respond strongly to the increase in cluster count.

\subsection*{Experiment 2}

\begin{table*}[ht]
\centering
\caption{Experiment 2: all uniformity measures under bulk cluster splits,
with and without the small auxiliary cluster.}
\label{tab:exp2}
\begin{tabular}{lrrrrrrrrr}
\toprule
& & & \multicolumn{2}{c}{Entropy} & \multicolumn{2}{c}{R\'{e}nyi-2} & \multicolumn{2}{c}{HHI} & \\
\cmidrule(lr){4-5}\cmidrule(lr){6-7}\cmidrule(lr){8-9}
Configuration & $K$ & MAS & $H$ & $H/\ln K$ & $H_2$ & $H_2/\ln K$ & $1-\mathrm{HHI}$ & $1-\mathrm{HHI}^*$ & $1-\mathrm{Gini}$ \\
\midrule
\multicolumn{10}{l}{\textit{With small cluster}} \\
$[4950,\, 4950,\, 100]$          & 3 & 0.9856 & 0.7422 & 0.6756 & 0.7130 & 0.6490 & 0.5099 & 0.7648 & 0.6767 \\
$[4950,\, 2475,\, 2475,\, 100]$  & 4 & 0.7925 & 1.0853 & 0.7829 & 1.0007 & 0.7218 & 0.6324 & 0.8431 & 0.6362 \\
$[2475{\times}4,\, 100]$         & 5 & 0.9945 & 1.4284 & 0.8875 & 1.4060 & 0.8736 & 0.7549 & 0.9436 & 0.8100 \\
\midrule
\multicolumn{10}{l}{\textit{Without small cluster}} \\
$[4950,\, 4950]$                 & 2 & 1.0000 & 0.6931 & 1.0000 & 0.6931 & 1.0000 & 0.5000 & 1.0000 & 1.0000 \\
$[4950,\, 2475,\, 2475]$         & 3 & 0.7917 & 1.0397 & 0.9464 & 0.9808 & 0.8928 & 0.6250 & 0.9375 & 0.8333 \\
$[2475{\times}4]$                & 4 & 1.0000 & 1.3863 & 1.0000 & 1.3863 & 1.0000 & 0.7500 & 1.0000 & 1.0000 \\
\bottomrule
\end{tabular}
\end{table*}

The numerical results are reported in Tables~\ref{tab:exp2}.

Unlike Experiment~1, this experiment changes the large-scale mass distribution. Accordingly, all measures move appreciably. The relevant question is how well those movements reflect the intermediate imbalance created by splitting only one of the two large clusters.

\paragraph{MAS.}
MAS decreases sharply when one large cluster is split:
\[
0.9856 \to 0.7925
\]
with the small cluster present, and
\[
1.0000 \to 0.7917
\]
without it. When both large clusters are split, the score returns close to its maximum:
\[
0.9945 \quad\text{and}\quad 1.0000,
\]
respectively. Thus MAS distinguishes clearly between the asymmetric intermediate partition and the balanced endpoint, while remaining only weakly affected by the $1\%$ cluster.

\paragraph{Entropy measures.}
Where the small cluster is present, all entropies, normalized and unnormalized, increase at each split. This is expected from their definitions: introducing additional categories increases diversity, even when the intermediate partition is significantly less balanced in bulk of the partition. The normalized forms reduce the dependence on absolute scale and correctly equal $1$ at perfectly uniform partitions, however they still increase despite greater imbalance in the bulk of the distribution. 

\paragraph{HHI-based measures.}
Similar to the entropy measures, both $1-\mathrm{HHI}$ and $1-\mathrm{HHI}^*$ increase across splits with the presence of the small cluster, with $1-\mathrm{HHI}$ increasing without the small cluster as well. This illustrates their dependence on $K$. With $1-\mathrm{HHI}$, an equal split with two primary clusters ([$4950, 4950$]) gives $0.5000$, whereas the equal split with four clusters ([$2475 \times 4$]) gives $0.7500$. Thus the maximum score depends on $K$. The normalized version $1-\mathrm{HHI}^*$ partially corrects this scale issue and equals $1$ at uniform partitions for any $K$. However the presence of a small auxiliary cluster deflates its normalizing factor $\frac{1}{K}$, inflating the score. As a result, it increases even as the bulk of the partition becomes more imbalanced.

\paragraph{Gini-based score.}
The quantity $1-\mathrm{Gini}$ distinguishes the intermediate partition from the balanced cases, but its sensitivities are inconsistent: In the presence of the small cluster, it decreases by $\approx 0.04$ when a single bulk cluster is split, yet increases by $\approx 0.13$ from its value on the un-split partition after both bulk clusters split. The score increases significantly more for distributions with similar bulk uniformity than it decreases for heavy imbalance in the bulk of the distribution.

\paragraph{Summary of Experiment 2.}
Experiment~2 indicates that MAS responds strongly when substantial mass is redistributed, reaches its maximum at balanced partitions regardless of $K$, and remains comparatively insensitive to a small auxiliary cluster. The comparison measures capture related aspects of the partition, but several of them either depend strongly on cluster count, reward category proliferation, or vary noticeably in the presence of low-mass clusters.

\subsection{Experiment 3}

\begin{table*}[t]
\centering
\caption{Experiment 3: PWRS of each composite scorer $\Phi(v)$ across eight datasets.
Bold indicates the highest PWRS on each dataset. The Null Reference ($v \equiv 1$)
measures the baseline contribution of the $K_{\mathrm{eff}}$-silhouette factors alone.}
\label{tab:exp3_pwrs}
\setlength{\tabcolsep}{5pt}
\begin{tabular}{lcccccccc}
\toprule
& \multicolumn{3}{c}{2D} & \multicolumn{2}{c}{4D} & 13D & 30D & 60D \\
\cmidrule(lr){2-4}\cmidrule(lr){5-6}\cmidrule(lr){7-7}\cmidrule(lr){8-8}\cmidrule(lr){9-9}
Scorer $v$ & Aggr. & Moons & Unbal. & Iris & Banknote & Wine & WDBC & Sonar \\
\midrule
Null Reference   & 0.600 & 0.886 & 0.815 & 0.841 & 0.690 & 0.886 & 0.956 & 0.595 \\
\midrule
MAS              & \textbf{0.822} & \textbf{0.932} & \textbf{0.852} & \textbf{0.886} & \textbf{0.857} & \textbf{0.977} & 0.911 & \textbf{0.690} \\
$1 - \text{Gini}$      & 0.422 & 0.909 & 0.630 & 0.864 & 0.643 & 0.909 & \textbf{0.978} & 0.643 \\
$1 - \text{HHI}$       & 0.800 & 0.364 & 0.759 & 0.682 & 0.381 & 0.864 & 0.444 & 0.595 \\
$1 - \text{HHI}^*$     & 0.756 & 0.886 & 0.759 & 0.864 & 0.738 & 0.909 & \textbf{0.978} & 0.619 \\
$H$ (Entropy)          & 0.667 & 0.091 & 0.519 & 0.159 & 0.262 & 0.114 & 0.022 & 0.357 \\
$H/\ln K$              & 0.667 & 0.886 & 0.685 & 0.864 & 0.690 & 0.886 & 0.956 & 0.619 \\
$H_2$ (R\'{e}nyi-2)    & 0.667 & 0.091 & 0.500 & 0.159 & 0.286 & 0.114 & 0.044 & 0.381 \\
$H_2/\ln K$            & 0.556 & 0.909 & 0.648 & 0.864 & 0.762 & 0.909 & \textbf{0.978} & 0.643 \\
$K_{\mathrm{eff}}$     & 0.622 & 0.091 & 0.500 & 0.159 & 0.286 & 0.114 & 0.044 & 0.357 \\
\bottomrule
\end{tabular}
\end{table*}

\begin{table*}[t]
\centering
\caption{Experiment 3: ARI of the partition selected by each composite scorer $\Phi(v)$
(i.e.\ the partition assigned the highest score). The ground-truth partition is included
in the candidate set on every dataset, so a score of $1.00$ indicates the scorer
successfully identified it as top-ranked.}
\label{tab:exp3_ari}
\setlength{\tabcolsep}{5pt}
\begin{tabular}{lcccccccc}
\toprule
& \multicolumn{3}{c}{2D} & \multicolumn{2}{c}{4D} & 13D & 30D & 60D \\
\cmidrule(lr){2-4}\cmidrule(lr){5-6}\cmidrule(lr){7-7}\cmidrule(lr){8-8}\cmidrule(lr){9-9}
Scorer $v$ & Aggr. & Moons & Unbal. & Iris & Banknote & Wine & WDBC & Sonar \\
\midrule
Null Reference   & 0.31 & 0.66 & 0.60 & 0.57 & 0.02 & 0.45 & 0.78 & -0.00 \\
\midrule
MAS              & \textbf{0.99} & \textbf{1.00} & \textbf{1.00} & \textbf{1.00} & \textbf{1.00} & \textbf{0.85} & \textbf{1.00} & \textbf{1.00} \\
$1 - \text{Gini}$      & 0.31 & \textbf{1.00} & 0.99 & 0.57 & \textbf{1.00} & \textbf{0.85} & \textbf{1.00} & \textbf{1.00} \\
$1 - \text{HHI}$       & 0.92 & 0.34 & \textbf{1.00} & 0.48 & 0.19 & \textbf{0.85} & 0.29 & 0.04 \\
$1 - \text{HHI}^*$     & \textbf{0.99} & 0.66 & 0.99 & 0.57 & \textbf{1.00} & \textbf{0.85} & \textbf{1.00} & -0.00 \\
$H$ (Entropy)          & 0.66 & 0.20 & 0.46 & 0.32 & 0.17 & 0.30 & 0.15 & 0.02 \\
$H/\ln K$              & 0.31 & 0.66 & 0.99 & 0.57 & \textbf{1.00} & \textbf{0.85} & 0.78 & -0.00 \\
$H_2$ (R\'{e}nyi-2)    & 0.66 & 0.20 & 0.46 & 0.32 & 0.17 & 0.30 & 0.15 & 0.02 \\
$H_2/\ln K$            & 0.71 & 0.66 & 0.99 & 0.57 & \textbf{1.00} & \textbf{0.85} & \textbf{1.00} & \textbf{1.00} \\
$K_{\mathrm{eff}}$     & 0.66 & 0.20 & 0.46 & 0.32 & 0.17 & 0.30 & 0.15 & 0.02 \\
\bottomrule
\end{tabular}
\end{table*}

The numerical results are reported in Tables~\ref{tab:exp3_pwrs}--\ref{tab:exp3_ari}.
The Null Reference establishes how much ranking quality the
$K_{\mathrm{eff}}$-silhouette factors alone provide; any uniformity
measure must improve upon it to justify its inclusion. The relevant
questions are whether each measure consistently improves over this
baseline, and whether the top-ranked partition under each scorer
corresponds to a high-quality clustering.

\paragraph{MAS.}
MAS achieves the highest PWRS on seven of eight datasets, and the
top-ranked partition under MAS attains an ARI of $1.00$ on six datasets
and at least $0.85$ on all eight. Comparatively, MAS achieves the maximum PWRS and ARI among baselines in all but one case, where the ground truth partition is genuinely non-uniform. 

\paragraph{Raw label-sensitive measures.}
The unnormalized measures $K_{\mathrm{eff}}$, $H$, and $H_2$ perform
poorly throughout, with PWRS falling below $0.5$ on several datasets
and the top-ranked partition rarely exceeding an ARI of $0.20$ outside
the Aggregation dataset. Because these measures increase with the
introduction of additional cluster labels regardless of mass
distribution, they systematically reward fragmented partitions and
often underperform even the Null Reference. This is consistent
with the behavior documented in Experiments~1 and~2.

\paragraph{Normalized measures.}
The normalized variants $H/\ln K$, $H_2/\ln K$, and $1-\mathrm{HHI}^*$
are more competitive, achieving high PWRS on several datasets and
selecting the ground-truth partition on a subset of them. However their
performance is uneven: $H/\ln K$ and $1-\mathrm{HHI}^*$ fail to improve
over the Null Reference on Sonar, and $H_2/\ln K$ selects a suboptimal
partition on Aggregation. Their sensitivity to the normalizing factor
$\ln K$ or $1/K$ introduces instability in the presence of small clusters and inconsistency in the penalization of imbalance, causing imbalanced influence from the other terms in $\phi$.

\paragraph{HHI-based and Gini-based scores.}
The unnormalized score $1-\mathrm{HHI}$ is inconsistent across datasets:
it achieves a PWRS of $0.800$ on Aggregation but falls to $0.364$ on
Moons and $0.444$ on WDBC, well below the Null Reference in both cases.
The score $1-\mathrm{Gini}$ is more stable but still underperforms MAS in both PWRS and ARI-accuracy of the highest rank configuration due to irregularities in its sensitivity to imbalance.

\paragraph{Summary of Experiment~3.}
Across eight datasets and 2--60 dimensions, MAS consistently improves over the Null Reference and other measures in both ranking quality and top-partition selection. Label-sensitive measures
degrade or invert the baseline ranking on multiple datasets, confirming
MAS's bulk-attentive self-consistency formulation can improve model selection. Normalized variants recover much of this loss
but remain susceptible to a bias toward inflated cluster count. MAS avoids
both failure modes, and its top-ranked partition coincides with or
closely approximates the ground-truth partition on every dataset in
the sweep.

\section{Conclusion}
The Mass Agreement Score (MAS) is a bounded measure of cluster size uniformity
derived from a point-centric perspective of distributional self-consistency. Rather
than evaluating uniformity purely as a property of cluster labels, MAS evaluates
the distribution as it is experienced by individual data points. Each cluster
contributes in proportion to the fraction of points it contains, and its size is
compared to the size-biased mean of the complementary distribution obtained by
excluding that cluster. These two structural choices, mass weighting and
self-consistency baselining, directly produce several desirable properties,
including sensitivity to dominant clusters, robustness to fragmentation of
negligible mass, invariance to relabeling, and scale invariance.

The experimental results illustrate the practical implications of this perspective.
Fragmenting a cluster containing only $1\%$ of the data into many pieces changes
MAS by less than $0.0003$, indicating that negligible mass has minimal influence on
the score. In contrast, redistributing substantial mass produces large and
interpretable changes: splitting a cluster containing roughly half the data lowers
MAS by $\approx 0.21$, reflecting the emergence of a strongly asymmetric partition
before balance is restored. In a composite hyperparameter scoring pipeline across
eight datasets spanning two to sixty dimensions, MAS consistently improves model selection decisions relative to all baselines in both
pairwise ranking quality and top-partition selection. Its performance illustrates how its fragmentation robustness and size-consistency notion of uniformity can guide optimization towards accurate recovery of ground truth.

Comparison with entropy-based and concentration-based measures clarifies where MAS
sits among existing tools. Shannon entropy is sensitive to the introduction of
low-probability clusters even when dominant mass structure is unchanged; HHI-style
measures avoid that sensitivity but remain dependent on cluster count in ways that
produce inconsistent scores for similarly uniform partitions. MAS avoids both
failure modes because its value is determined entirely by relative cluster sizes,
not by the number of labels used to represent them.

MAS therefore addresses a question distinct from that of classical diversity
measures: not how evenly mass is spread across categories, but how consistently
clusters are sized from the perspective of the points within them.

Several directions remain open. The pipeline experiment demonstrates that MAS
contributes reliable marginal value as a uniformity component in a composite scorer,
and its bounded range and interpretable behavior under redistribution make it a
natural candidate for further exploration in multistage or multiobjective
hyperparameter optimization. Beyond clustering, MAS's distributional
self-consistency construction could be applied to any probability distribution over
discrete categories, suggesting potential use wherever value consistency of a
partition is of interest.

\bibliographystyle{plain}
\bibliography{references}

\appendices
\section{Key Properties}
\label{app:key_props}
In this section, I prove several properties of the Mass Agreement Score, including maximization strictly at equal cluster sizes, scale invariance, and stability under point reassignment.

\subsection{Maximum Achieved at Equal Cluster Sizes}
\begin{proposition}
$U = 1$ if and only if all clusters have equal size.
\end{proposition}

\begin{proof}
$(\Leftarrow)$
Since $n_i = N/K$ for all $i$,  
\[
S_i = \frac{K(N/K)^2 - (N/K)^2}{N - N/K} = \frac{(K-1)(N/K)^2}{(K - 1)N/K} =
\]
\[
N/K = n_i.
\]
Thus, all agreement terms $1 - \frac{|n_i - S_i|}{N} = 1$, giving $U = 1$.

$(\Rightarrow)$ Since $U = 1$ and the weights sum to $1$ with each factor $\leq 1$, equality forces $|n_i - S_i| = 0$ for every $i$ with $n_i > 0$, meaning $n_i = S_i$. 
\[
  S_i = n_i \implies \frac{\sum_j n_j^2 - n_i^2}{N - n_i} = n_i \implies 
\]
\[
\sum_j n_j^2 - n_i^2 = n_i(N - n_i) = Nn_i - n_i^2 \implies 
\]
\[
\sum_j n_j^2 = Nn_i \implies n_i = \frac{\sum_j n_j^2}{N}
\]
Since the right-hand side is a constant, all $n_i$ are equal.
\end{proof}

\subsection{Shrinkage to Zero with Dominance}
\begin{proposition}
\label{prop:2_cluster_dominance}
Consider the dominated two-cluster partition $(N-1,1)$. As $N \to \infty$, the Mass Agreement Score satisfies
\[
U \to 0 .
\]
\end{proposition}

\begin{proof}
Let $n_1 = N-1$ and $n_2 = 1$. The squared-size sum is
\[
Q = (N-1)^2 + 1.
\]

For the large cluster,
\[
S_1 = \frac{Q - (N-1)^2}{N-(N-1)} = \frac{1}{1} = 1,
\]
so
\[
\frac{|n_1 - S_1|}{N} = \frac{N-2}{N}.
\]

For the singleton cluster,
\[
S_2 = \frac{Q - 1}{N-1} = \frac{(N-1)^2}{N-1} = N-1,
\]
so
\[
\frac{|n_2 - S_2|}{N} = \frac{N-2}{N}.
\]

Thus both clusters have the same agreement term
\[
1 - \frac{N-2}{N} = \frac{2}{N}.
\]

Since the weights sum to $1$, the Mass Agreement Score is
\[
U = \frac{2}{N}.
\]

Therefore $U \to 0$ as $N \to \infty$.
\end{proof}

\subsection{Cluster-Confined Reassignment Stability}
\begin{proposition}
Let $\mathcal{S} \subseteq [K]$ be a set of \emph{active} clusters with total mass
\[
m = \sum_{i \in \mathcal{S}} n_i = \mu N,
\qquad \mu \in [0,1].
\]
Assume that all point reassignments are confined to $\mathcal{S}$, so that
clusters in $\mathcal{L} := [K]\setminus \mathcal{S}$ are untouched. Then
\[
|\Delta U| \;\le\; \mu(2-\mu).
\]
\end{proposition}

\begin{proof}
Let $\mathcal{L} = [K]\setminus \mathcal{S}$ denote the complementary set of
\emph{inactive} clusters. Since no reassignment involves clusters in
$\mathcal{L}$, their sizes $n_i$ remain fixed.

Decompose the utility shift as
\[
\Delta U = \Delta U_{\mathcal{S}} + \Delta U_{\mathcal{L}}.
\]

\medskip
\noindent\textbf{Contribution from the active clusters.}
Because the total weight of clusters in $\mathcal{S}$ is exactly $\mu$, both
$U_{\mathcal{S}}$ and $U_{\mathcal{S}}'$ lie in the interval $[0,\mu]$.
Indeed, the coefficients $n_i/N$ over $i\in\mathcal{S}$ sum to $\mu$, while
each corresponding agreement term is at most $1$. Hence
\[
|\Delta U_{\mathcal{S}}|
\;\le\;
\mu.
\]

\medskip
\noindent\textbf{Contribution from the inactive clusters.}
For each $i \in \mathcal{L}$, the weight $n_i/N$ is unchanged, so only the
agreement term can vary. Since
\[
S_i = \frac{Q - n_i^2}{N - n_i},
\]
and $n_i$ is fixed for $i\in\mathcal{L}$, we have
\[
|\Delta S_i|
=
\frac{|\Delta Q|}{N-n_i}.
\]
Therefore
\[
|\Delta a_i|
=
\frac{|\Delta |n_i-S_i||}{N}
\;\le\;
\frac{|\Delta S_i|}{N}
=
\frac{|\Delta Q|}{N(N-n_i)}.
\]

It remains to control $|\Delta Q|$. Since all reassignments are confined to
$\mathcal{S}$, the only part of
\[
Q = \sum_{j=1}^K n_j^2
\]
that can change is
\[
Q_{\mathcal{S}} = \sum_{j\in\mathcal{S}} n_j^2.
\]
Because the total mass in $\mathcal{S}$ is $m$, we have
\[
Q_{\mathcal{S}}
\le
\Bigl(\sum_{j\in\mathcal{S}} n_j\Bigr)^2
=
m^2,
\]
and the same bound holds after reassignment. Thus
\[
|\Delta Q|
=
|Q_{\mathcal{S}}' - Q_{\mathcal{S}}|
\le
m^2
=
\mu^2 N^2.
\]

Substituting this into the bound for the inactive contribution gives
\[
|\Delta U_{\mathcal{L}}|
\;\le\;
\sum_{i\in\mathcal{L}} \frac{n_i}{N}\, |\Delta a_i|
\;\le\;
\mu^2 \sum_{i\in\mathcal{L}} \frac{n_i}{N-n_i}.
\]

Now the function $x \mapsto x/(N-x)$ is increasing and convex on $(0,N)$, so
under the constraint
\[
\sum_{i\in\mathcal{L}} n_i = (1-\mu)N,
\]
the sum
\[
\sum_{i\in\mathcal{L}} \frac{n_i}{N-n_i}
\]
is maximized when all inactive mass is concentrated in a single cluster. This
yields
\[
\sum_{i\in\mathcal{L}} \frac{n_i}{N-n_i}
\;\le\;
\frac{(1-\mu)N}{N-(1-\mu)N}
=
\frac{1-\mu}{\mu}.
\]
Hence
\[
|\Delta U_{\mathcal{L}}|
\;\le\;
\mu^2 \cdot \frac{1-\mu}{\mu}
=
\mu(1-\mu).
\]

\medskip
\noindent\textbf{Conclusion.}
Combining the active and inactive contributions,
\[
|\Delta U|
\;\le\;
|\Delta U_{\mathcal{S}}| + |\Delta U_{\mathcal{L}}|
\;\le\;
\mu + \mu(1-\mu)
=
\mu(2-\mu).
\]
\end{proof}

\subsection{General Reassignment Stability}
\begin{proposition}
For $K \ge 2$ nonempty clusters, $\epsilon \in [0,1]$, reassigning $\epsilon N$ points perturbs
the score by at most $3\epsilon$:
\[
  |\Delta\,U| \;<\; 3\epsilon.
\]
\end{proposition}

\begin{proof}
It suffices to show that a single point reassignment perturbs $U$ by at most
$3/N$; the general bound then follows by the triangle inequality applied
$\epsilon N$ times, giving $\epsilon N \cdot (3/N) = 3\epsilon$.

\medskip
\noindent\textbf{Setup.}
Move one point from cluster $A$ to cluster $B$, where
$n_A \ge 1$ and $n_B \ge 0$. Clusters with $n_i = 0$ contribute weight $n_i/N = 0$ to $U$ and contribute
$n_i^2 = 0$ to $Q = \sum_k n_k^2$, so they are inert and we may freely allow
$n_B \ge 0$. 

The total $N$ is unchanged. The induced change
in $Q$ is
\[
  \Delta Q
  = (n_A-1)^2 + (n_B+1)^2 - n_A^2 - n_B^2
  = 2(n_B - n_A + 1).
\]

Write $a_i = 1 - |n_i - S_i|/N \in [0,1]$ for the agreement term of cluster
$i$, and $a_i'$ for its value after the move. The change in $U$ is
\[
  \Delta U
  = \sum_{\text{all } i}
    \left[
      \frac{n_i'}{N}a_i' - \frac{n_i}{N}a_i
    \right].
\]

\medskip
\noindent\textbf{Rearranging the endpoint terms.}
Separate bystanders and the two affected clusters:
\[
\begin{aligned}
\Delta U
&= \frac{1}{N}\Bigg[
    \sum_{i \notin \{A,B\}} n_i(a_i' - a_i) \\
&\qquad + (n_A-1)a_A' - n_A a_A \\
&\qquad + (n_B+1)a_B' - n_B a_B
\Bigg].
\end{aligned}
\]

Rewrite the last two terms:
\[
(n_A-1)a_A' - n_A a_A
= (n_A-1)(a_A' - a_A) - a_A,
\]
\[
(n_B+1)a_B' - n_B a_B
= n_B(a_B' - a_B) + a_B'.
\]

Thus
\[
\begin{aligned}
\Delta U
&= \frac{1}{N}\Bigg[
    \sum_{i \notin \{A,B\}} n_i(a_i' - a_i) \\
&\qquad + (n_A-1)(a_A' - a_A) \\
&\qquad + n_B(a_B' - a_B) \\
&\qquad + (a_B' - a_A)
\Bigg].
\end{aligned}
\]

\medskip
\noindent\textbf{Bystanders $i \notin \{A,B\}$.}
Since $n_i' = n_i$, the weight cancels and only the agreement term changes.
The only change to
\[
  S_i = \frac{Q - n_i^2}{N - n_i}
\]
comes through $\Delta Q$, giving
\[
  |\Delta S_i|
  = \frac{|\Delta Q|}{N - n_i}
  = \frac{2|n_B - n_A + 1|}{N - n_i}
  \le \frac{2|n_B - n_A + 1|}{n_A + n_B}.
\]
Since $n_A \ge 1$ and $n_B \ge 0$, one has
$|n_B - n_A + 1| < n_A + n_B$, hence $|\Delta S_i| < 2$.
Therefore $|\Delta|n_i - S_i|| \le |\Delta S_i| < 2$, giving
\[
  |a_i' - a_i| < \frac{2}{N}.
\]
The bystander contribution thus satisfies
\[
  \left|
  \frac{1}{N}
  \sum_{i \notin \{A,B\}} n_i(a_i' - a_i)
  \right|
  \le
  \frac{2(N-n_A-n_B)}{N^2}.
\]

\medskip
\noindent\textbf{Cluster $A$.}

Define $p = Q - n_A^2$ (the contribution to $Q$ from all non-$A$ clusters)
and $d = N - n_A$ (the total mass outside $A$), so that $S_A = p/d$.

After the move, $Q' = Q + 2(n_B - n_A + 1)$ and $n_A' = n_A - 1$.
Computing the numerator of $S_A'$ directly:
\[
  Q' - (n_A-1)^2
  = (Q - n_A^2) + (2n_A - 1) + 2(n_B - n_A + 1) = 
\]
\[
p + 2n_B + 1,
\]
and the denominator is $N - n_A' = d + 1$. Therefore
\[
  \Delta S_A
  = \frac{p + 2n_B + 1}{d+1} - \frac{p}{d}
  = \frac{d(2n_B+1) - p}{d(d+1)}.
\]
We bound the numerator using two constraints: $p \ge n_B^2$ (cluster $B$
alone contributes $n_B^2$ to $p$) and $p \le d^2$ (all non-$A$ mass is at
most $d$, so $\sum_{i \ne A} n_i^2 \le (\sum_{i \ne A} n_i)^2 = d^2$).

\emph{Upper bound.} Using $p \ge n_B^2$ and completing the square:
\[
  d(2n_B+1) - p
  \;\le\;
\]
\[
d(2n_B+1) - n_B^2 = -\bigl(n_B^2 - 2d\,n_B + d^2\bigr) + d^2 + d = 
\]
\[
 d(d+1) - (n_B - d)^2
  \;\le\;
  d(d+1),
\]
so $\Delta S_A \le 1$.

\emph{Lower bound.} Using $p \le d^2$ and $n_B \ge 0$:
\[
  d(2n_B+1) - p
  \;\ge\;
  d - d^2
  = -d(d-1)
  > -d(d+1),
\]
where the last inequality reduces to $2d > 0$, which holds since $d \ge 1$.
Hence $\Delta S_A > -1$.

Combining, $|\Delta S_A| \le 1$. Since $|\Delta n_A| = 1$, the argument of
the absolute value in $a_A$ shifts by at most $|\Delta n_A| + |\Delta S_A|
\le 2$, giving
\[
  |a_A' - a_A| \;\le\; \frac{2}{N}.
\]
Hence
\[
  \left|\frac{n_A-1}{N}(a_A' - a_A)\right|
  \;\le\;
  \frac{2(n_A-1)}{N^2}.
\]

\medskip
\noindent\textbf{Cluster $B$.}

Define $q = Q - n_B^2$ and $e = N - n_B$, so that $S_B = q/e$.
After the move, $Q' = Q + 2(n_B - n_A + 1)$ and $n_B' = n_B + 1$.
Computing the numerator of $S_B'$:
\[
  Q' - (n_B+1)^2
  = (Q - n_B^2) + 2(n_B - n_A + 1) - 2n_B - 1 =
\]
\[
 q - (2n_A - 1),
\]
and the denominator is $N - n_B' = e - 1$. Therefore
\[
  \Delta S_B
  = \frac{q - (2n_A-1)}{e-1} - \frac{q}{e}
  = \frac{q - e(2n_A-1)}{e(e-1)}.
\]
We bound the numerator using $q \ge n_A^2$ (cluster $A$ contributes $n_A^2$
to $q$) and $q \le e^2$, with $1 \le n_A \le e - 1$ (since
$n_B \le N - 2$ is required for $e - 1 \ge 1$; if $n_B = N-1$ then
$U = 2/N$ by~\ref{prop:2_cluster_dominance}, and after the move all $N$ points lie in $B$, meaning $U' = 0$ and $|\Delta U| = 2/N < 3/N$).

\emph{Upper bound.} Using $q \le e^2$ and $n_A \ge 1$:
\[
  q - e(2n_A-1)
  \;\le\;
  e^2 - e(2n_A - 1)
  = e\bigl(e - 2n_A + 1\bigr)
  \;\le\;
\]
\[
e(e-1).
\]
Hence $\Delta S_B \le 1$.

\emph{Lower bound.} Using $q \ge n_A^2$ and completing the square:
\[
  q - e(2n_A-1)
  \;\ge\;
\]
\[
n_A^2 - e(2n_A-1) = \bigl(n_A^2 - 2e\,n_A + e^2\bigr) - e^2 + e = 
\]
\[
(n_A - e)^2 - e(e-1)
  \;\ge\;
  -e(e-1),
\]
so $\Delta S_B \ge -1$.

Thus $|\Delta S_B| \le 1$, and by the same argument as for cluster $A$:
\[
  \left|\frac{n_B}{N}(a_B' - a_B)\right|
  \;\le\;
  \frac{2n_B}{N^2}.
\]

\medskip
\noindent\textbf{Endpoint weight term.}
Since $a_A,a_B' \in [0,1]$,
\[
\left|
\frac{a_B' - a_A}{N}
\right|
\le
\frac{1}{N}.
\]

\medskip
\noindent\textbf{Assembly.}
Combining the agreement-change terms,
\[
\frac{2}{N^2}\Big((N-n_A-n_B)+(n_A-1)+n_B\Big) =
\]
\[
\frac{2(N-1)}{N^2}.
\]
Adding the endpoint weight contribution gives
\[
|\Delta U|
\le
\frac{2(N-1)}{N^2} + \frac{1}{N}
=
\frac{3N-2}{N^2}
<
\frac{3}{N}.
\]

Thus a single reassignment changes $U$ by at most $3/N$, and applying this
bound over $\epsilon N$ moves yields
\[
|\Delta U| < 3\epsilon.
\]
\end{proof}

\subsection{Scale Invariance}
\begin{proposition}
Replacing each $n_i$ with $r n_i$ for any positive real number $r$ leaves $U$ unchanged.
\end{proposition}

\begin{proof}
Under the scaling $n_i \mapsto r n_i$, the total becomes $N \mapsto rN$ and $\sum_j n_j^2 \mapsto r^2 \sum_j n_j^2$. The leave-one-out baseline transforms as:
\[
  S_i \mapsto \frac{r^2(\sum_j n_j^2 - n_i^2)}{r(N - n_i)} = r S_i.
\]
Each term in $U$ then becomes:
\[
  \frac{r n_i}{r N}\left(1 - \frac{|r n_i - r S_i|}{r N}\right) = \frac{n_i}{N}\left(1 - \frac{|n_i - S_i|}{N}\right).
\]
The $r$ factors cancel exactly, term by term.
\end{proof}

\end{document}